\documentclass{article}
\usepackage[english]{babel}
\usepackage[letterpaper,top=2cm,bottom=2cm,left=3cm,right=3cm,marginparwidth=1.75cm]{geometry}

\usepackage{amsmath}
\usepackage{graphicx}
\usepackage{xcolor}
\usepackage{amssymb}
\usepackage{amsthm}
\usepackage[colorlinks=true, allcolors=blue]{hyperref}
\usepackage{bbm}

\newcommand{\bR}{\mathbb{R}}
\newcommand{\bP}{\mathbb{P}}
\newcommand{\bE}{\mathbb{E}}

\newcommand{\ip}[2]{\left\langle#1,#2\right\rangle}
\newcommand{\norm}[1]{\left\Vert#1\right\Vert}
\newcommand{\abs}[1]{\left\vert#1\right\vert}
\newcommand{\absip}[2]{\left\vert\left<#1,#2\right>\right\vert}
\newcommand{\beq}{\begin{equation}}
\newcommand{\eeq}{\end{equation}}

\newcommand{\SRG}{\operatorname{SRG}}
\newcommand{\ReLU}{\operatorname{ReLU}}
\newcommand{\sign}{\operatorname{sign}}

\definecolor{dkred}{rgb}{0.33, 0, 0}
\definecolor{peachy}{rgb}{1, 0.67, 0.67}

\theoremstyle{plain}
\newtheorem{theorem}{Theorem}
\newtheorem{thm}[theorem]{Theorem}
\newtheorem{lemma}[theorem]{Lemma}
\newtheorem{lem}[theorem]{Lemma}

\newtheorem{prop}[theorem]{Proposition}

\newtheorem{cor}[theorem]{Corollary}

\theoremstyle{definition}
\newtheorem{definition}[theorem]{Definition}

\newtheorem{rem}[theorem]{Remark}

\title{Towards a mathematical theory of superposition}
\author{Michael I. Ivanitskiy\thanks{Department of Applied Mathematics and Statistics, Colorado School of Mines, Golden, CO, USA}
\and John Jasper\thanks{Department of Mathematics \& Statistics, Air Force Institute of Technology, Wright-Patterson AFB, OH, USA
}
\and
Emily J.\ King\thanks{Corresponding author: \texttt{emily.king@colostate.edu}} \thanks{Department of Mathematics, Colorado State University, Fort Collins, CO, USA} 
\and Dustin G.\ Mixon\thanks{Department of Mathematics, The Ohio State University, Columbus, OH, USA} \thanks{Translational Data Analytics Institute, The Ohio State University, Columbus, OH}}

\date{}

\begin{document}
\maketitle

\begin{abstract}
We develop a mathematical theory of superposition in neural networks using tools from frame theory and compressed sensing. In our model, a sparse binary vector \(x\) of active features is encoded through an overcomplete dictionary \(W\), and feature recovery is performed by applying \(\operatorname{ReLU}(W^\top W x+b)\) with an appropriate bias vector \(b\). We prove several recovery theorems for this model. In the random-support setting, we establish high-probability support recovery for nearly tight, low-coherence dictionaries, with guarantees when the expected sparsity is up to order \(d/\log n\). In the worst-case support setting, we give a sharp and computable criterion for which sparsity levels permit support recovery. We apply this criterion to Gaussian random matrices and equiangular tight frames. For real equiangular tight frames with \(n>d+1\), we determine the exact recovery threshold in terms of the coherence. The proof of this result for real equiangular tight frames relies on a novel characterization---which should be of independent interest to frame theorists---of the distribution of signs in the Gram matrix.
\end{abstract}

\section{Introduction}

How does the human brain store information? 
A naive theory might posit that a single neuron is responsible for responding to a specific complex stimulus, e.g., one for the concept of \textit{grandmother} and another for \textit{Jennifer Aniston}.
This is a highly complex model since every feature requires a different neuron, and one would not expect such a model to efficiently generalize to new but related input.
The concept of a \textit{grandmother cell} originated as a joke \cite{Barwich19,gross2002genealogy}, building off the slightly less absurd idea of a mother cell. 
Meanwhile, the experiments that led to the notion of a \textit{Jennifer Aniston neuron} actually showed that neurons represent multiple concepts at the same time~\cite{quiroga2013brain}, as the same neuron responded to pictures of Jennifer Aniston and Lisa Kudrow, who costarred with Jennifer Aniston in the sitcom \textit{Friends}. 
A more plausible theory is that of \emph{distributed representation} (e.g.,~\cite{thorpe1995local,plate2002dist}), where concepts are represented as combinations of neuron activations.  
Distributed representation has also been observed in artificial neural networks.  
The goal of this paper is to provide some theoretical underpinning---through the lenses of frame theory and compressed sensing---to the observed behavior in artificial neural networks that is analogous to distributed representation in biological neural networks.

\subsection{Superposition in neural networks}

Artificial neural networks---which we will refer to simply as ``neural networks'' in what follows---consist of cascading affine linear and non-linear maps applied to vectors.
To train a neural network, one fixes an architecture (which specifies the form of each affine linear map, along with the intermediate nonlinearities), and then optimizes the constituent affine linear maps to fit a given training set.
Each output coordinate of each affine linear map is known as a \textit{neuron}, and the value that a neuron takes in response to a given input is known as its \textit{activation}.
The vector of all neuron activations resides in a vector space we call \textit{activation space}; a neural network with $d$ neurons has activation space $\mathbb{R}^d$.

Since the neural network was tuned to optimally transform training data for a given task, one might expect the resulting activation space to be semantically meaningful.
This mimics the setting of word embeddings, where words are represented as vectors in $\bR^d$, and the resulting linear geometry captures semantic meaning~\cite{LevyG14,MikolovYZ13,PenningtonSM14}.
Following~\cite{Anthropic22}, we will assume that semantic meaning can be captured by \textit{sparse coding}, i.e., there exists a dictionary of \textit{feature representations} $w_1,\ldots,w_n\in\mathbb{R}^d$ such that every realizable vector $h$ of neuron activations is well approximated by a nonnegative combination of at most $k$ feature representations:
\[
h
\approx \sum_{j\in K}x_jw_j,
\qquad
K\subseteq\{1,\ldots,n\},
\qquad|K|\leq k,
\qquad x_j\geq0.
\]
Notably, in the $n\gg d$ setting in which there are more features than neurons, the pigeonhole principle precludes the naive theory of each feature determining its own neuron; instead, neurons are forced to be \textit{polysemantic} by responding to a \textit{superposition} of features. Recent work has shown that features with interpretable semantic meaning sometimes exhibit more complicated geometry than single linear directions (i.e., $w_i$); however, sparse coding can still be leveraged to find local structure of feature space or higher dimensional irreducible features~\cite{engels2025not,geiger2026worldinside,bhalla2026saes}.

These notions of polysemanticity and superposition have recurred throughout the recent literature~\cite{Anthropic22,AroraLLMR18, Goh16, OlahCSG20,Anthropic22b}, as there is now ample evidence of superposition emerging from machine learning algorithms ``in the wild.'' 
Words with multiple meanings tend to embed as a superposition of \emph{discourse atoms}, which themselves are not embeddings of words~\cite{AroraLLMR18}; for example, the word ``tie'' is a superposition of discourse atoms, the most important of which are close to embeddings of words pertaining to clothing and competition records. 
To help decipher modern large language models (LLMs) like Pythia-70M and Gemma-2-2B, sparse autoencoders were used in~\cite{marks2024sparse} to decompose points in latent space as a superposition of human-interpretable concepts. 
Interestingly, one can exploit this compressibility of neural networks to squeeze even more functionality into a fixed architecture, i.e., training it to handle multiple unrelated tasks~\cite{cheung2019superposition}, thereby achieving an even higher level of polysemanticity. This property of compressibility may be fundamental to explaining why neural networks, both biological and artificial, are so successful at representing and performing computation over large amounts of information.
To better understand these emerging phenomena, a team from Anthropic developed so-called \textit{toy models of superposition} that they tested and analyzed under various assumptions~\cite{Anthropic22}.
Observations from that paper inspired the mathematical model we analyze in this paper.

We summarize the parameters of superposition as
\[
\begin{aligned}
\text{ \# active features }
&\ll \text{ \# neurons }&&\qquad\text{(sparsity)}\\
&\ll \text{ \# features }&&\qquad\text{(polysemanticity)}\\
&\ll \text{ \# patterns of active features }&&\qquad\text{(expressivity),}
\end{aligned}
\]
where $k$ is the maximum number of active features (\textit{sparsity}), $d$ is the number of neurons, $n$ is the number of features, and $\sum_{j\leq k}\binom{n}{j}$ is the number of patterns of active features.

As a simple illustration, suppose we only had two dimensions to represent color. 
(See Figure~\ref{fig:colors}.)
First, we could represent red as $(1,0)$ and blue as $(0,1)$.
Then any purple like $(1,1)$ would correspond to a superposition of the red and blue features.  
However, by restricting each feature vector to be a standard basis element (akin to the naive one-feature-per-neuron model), we limit the representable palette, e.g., no shade of green can be represented. 
As an alternative, one could choose a polysemantic model by assigning red to $(1,0)$, green to $(-1/2, \sqrt{3}/2)$, and blue to $(-1/2, -\sqrt{3}/2)$. 
Notably, this choice involves more features (colors, $n=3$) than neurons (dimensions, $d=2$). 
Now consider an arbitrary color represented as $(r,g,b)$, where $0 \leq r,g,b\leq 1$ give the amount of red, green, and blue in the color, respectively. 
(N.B., smaller numbers are darker.)
Such a color is represented in this model by the $2$-dimensional vector
\[
r \begin{pmatrix}1\\0\end{pmatrix}
+ g 
\begin{pmatrix}-1/2\\\sqrt{3}/2\end{pmatrix}
+ b 
\begin{pmatrix}-1/2\\-\sqrt{3}/2\end{pmatrix}.
\]
While every color is represented as a linear combination of the three feature vectors, certain colors are represented by the same vectors in $\bR^2$. 
For example, a dark red color with $(r,g,b)=(1/3,\,0,\,0)$ has the same representation as the medium salmon color with $(r,g,b)=(1,\,2/3,\,2/3)$:
\[
\underbrace{1/3\begin{pmatrix}1\\0\end{pmatrix} + 0\begin{pmatrix}-1/2\\\sqrt{3}/2\end{pmatrix} + 0\begin{pmatrix}-1/2\\-\sqrt{3}/2\end{pmatrix}}_{\text{rep'n of \fcolorbox{dkred}{dkred}{\textcolor{white}{dark red}}}}
~~ = ~~
\begin{pmatrix}1/3\\0\end{pmatrix}
~~ = ~~
\underbrace{1\begin{pmatrix}1\\0\end{pmatrix} + 2/3\begin{pmatrix}-1/2\\\sqrt{3}/2\end{pmatrix} + 2/3\begin{pmatrix}-1/2\\-\sqrt{3}/2\end{pmatrix}}_{\text{rep'n of \fcolorbox{peachy}{peachy}{medium salmon}}}.
\]
This is where feature sparsity comes into play: 
If we only allow colors that are nonnegative combinations of at most $k=2$ of the vectors, then the $2$-dimensional representation uniquely determines the underlying color.  
On the other hand, salmon and even white cannot be represented in this model since red, green, and blue are all present in these colors.\footnote{A color theorist will note that in this model, every hue is represented by a color of maximum saturation within the hue equivalence class.}

Note that the feature vectors in this second model point to the vertices of an equiangular triangle centered at the origin and are thus spread apart as far as possible for $3$ unit vectors in $\bR^2$.
In some quantitative sense, these feature vectors are as independent as possible.  
We will revisit this geometry in the next section.

\begin{figure}[htp!]
\includegraphics[width=\textwidth]{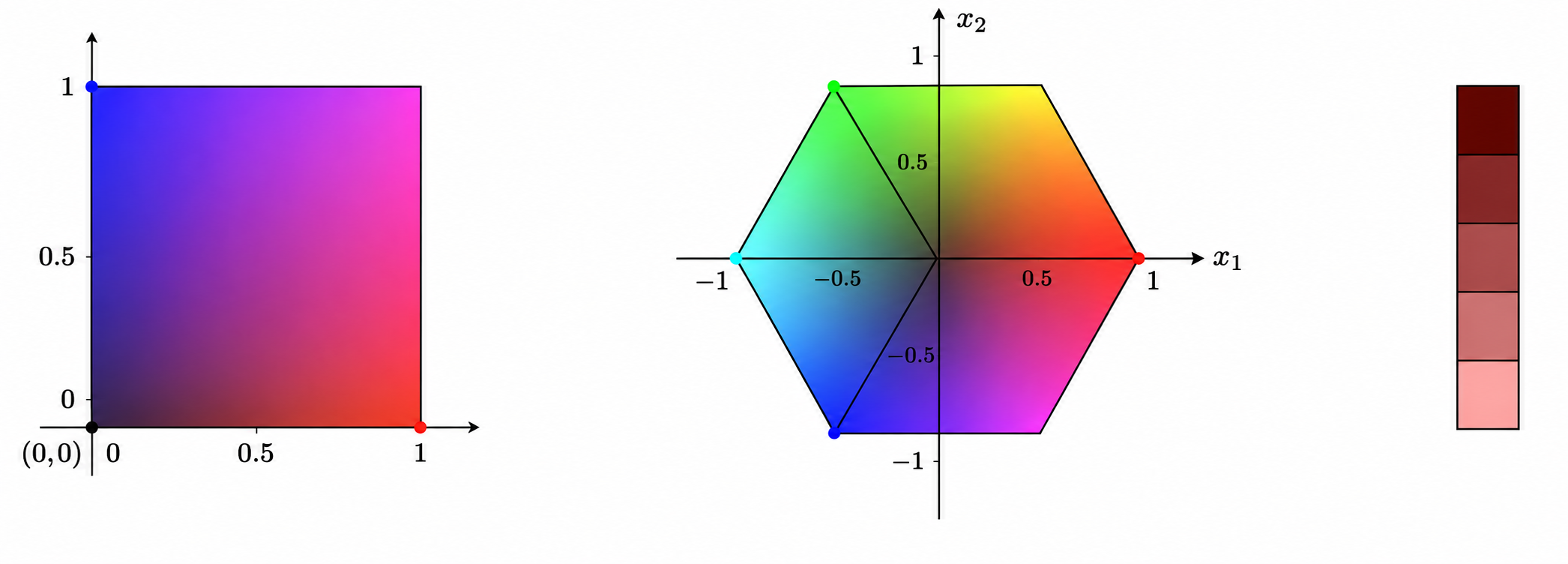}
\caption{\textbf{Left:} Every color represented by nonnegative linear combinations of the vectors $(1,0)$ and $(0,1)$, which represent red and blue, respectively.
\textbf{Middle:} Every color represented by nonnegative linear combinations of at most two of the vectors $(1,0)$, $(-1/2, \sqrt{3}/2)$, and $(-1/2, -\sqrt{3}/2)$, which represent red, green, and blue, respectively.
\textbf{Right:} Different saturations of the same hue that are all represented by the same vector as a (in general not $2$-sparse) linear combination of $(1,0)$, $(-1/2, \sqrt{3}/2)$, and $(-1/2, -\sqrt{3}/2)$, namely, $(1/3,0)$. 
 A $2$-sparsity assumption ensures unique representation, but also that not all colors are representable. 
In this case, only the darkest shade in the color bar is representable. An interactive version can be viewed at \protect\url{https://miv.name/web-tools/color-superposition.html}}
\label{fig:colors}
\end{figure}

\subsection{Frames and sparsity insist on themselves}

There are two key assumptions in the model of superposition from~\cite{Anthropic22}:
(1) there are more features than neurons, and 
(2) the features are sparsely active. 
These concepts are naturally viewed through the lenses of frame theory and compressed sensing. 
Frames are overcomplete generalizations of orthonormal bases that were originally introduced as a generalization of Fourier series~\cite{duffin1952class} and have connections to open problems in quantum information theory and combinatorial design theory (see, e.g., \cite{zauner1999grund,zauner2011found,gillespie2019equi}).  
Special highly structured frames (specifically, \emph{equiangular tight frames} defined below) have appeared as the final weight matrix of certain overtrained neural networks~\cite{papyan2020prev}. 
By virtue of their overcompleteness, frames allow for redundant representations of data, and this makes them a natural fit for analyzing superposition.  
Relatedly, compressed sensing, dictionary learning, and other sparsity-based methods concern the measurement and analysis of data which are sparse linear combinations of special vectors. See, e.g.,~\cite{waldron2018intro,foucart2013math,elad2010sparse} for general references of finite frame theory and sparsity-based methods.

Recall the three vectors from our simple example that represented colors as points in $\bR^2$ (see Figure~\ref{fig:colors}).
This example worked out well thanks to certain nice properties of these vectors. 
Consider a matrix $W \in \bR^{d \times n}$ with $n \geq d$. 
We say the columns $(w_i)_{i=1}^n$ of $W$ form a \emph{tight frame} if 
\beq\label{eqn:tightframe}
W W^\top = \sum_{i=1}^n w_i w_i^\top =  A I_d
\eeq
for some \emph{frame bound} $A > 0$. 
Tight frames are desirable because they give ``painless'' reconstruction formulas, i.e., for every $x \in \bR^d$, it holds that
\[
x  
= \frac{1}{A} \sum_{i=1}^n \ip{x}{w_i}w_i.
\]
If a tight frame $W$ is such that $\norm{w_i}=1$ and $\absip{w_i}{w_j}$ is constant for $i \neq j$, then it is an \emph{equiangular tight frame (ETF)}. 
Relatedly, $W$ is said to form a \emph{(sparse) dictionary} if data of interest is well-approximated by nonnegative linear combinations of a few of the columns $w_i$.
In our toy color example, the three feature vectors formed an ETF with $A = 3/2$ and with different vectors having absolute inner product $1/2$. This $W$ was also a sparse dictionary for the primary and secondary colors.  
If $W$ is such that the columns $(w_i)_{i=1}^n$ have unit norm (with no other assumed structure), then the \emph{coherence} of $W$ is
\[
\mu(W) = \max_{i \neq j} \absip{w_i}{w_j}.
\]
The \emph{Welch--Rankin bound} \cite{Ran55,Welch} is
\[
\mu(W) \geq \sqrt{\frac{n-d}{d(n-1)}}.
\]
This bound is saturated precisely when $W$ is an ETF.  
This means that ETFs consist of vectors that are, in some sense, as close as possible to being pairwise orthogonal given the fact that they are overcomplete. 
Note that if the columns of $W$ form an ETF, then the Gram matrix $W^\top W$ has rank $d$, the diagonal entries all equal $1$, and the off-diagonal entries have modulus equal to the Welch--Rankin bound, i.e., $\sim 1/\sqrt{d}$; that is, the Gram matrix is a low-rank entry-wise approximation to the identity matrix.
We call an ETF \emph{centered} if the all--ones vector $\mathbbm{1}$ is in the kernel of $W^\top W$~\cite{fickus2018equi} and a \emph{simplex ETF} if $n=d+1$.
There exists an ETF $W \in \bR^{d \times n}$ if and only if there also exists an ETF $V \in \bR^{(n-d)\times n}$ with $\sign(\ip{w_i}{w_j})=-\sign(\ip{v_i}{v_j})$.  
Any such $V$ is called a \emph{Naimark complement} of $W$~\cite{neumark1943representation,han2000frames}. 
Further, given a vector $x \in \bR^d$, we let $\norm{x}_0$ denote the cardinality of its support, i.e., its \emph{sparsity}.
(Sadly, this standard choice of nomenclature means that $x$ is \textit{more sparse} when its sparsity $\|x\|_0$ is \textit{smaller}.)
Finally, non-linear maps called \emph{activation functions} are critical components of the architecture of neural networks.
A common choice of activation function is ReLU, which sparsifies a vector by returning $\max\{x_i,0\}$ for each coordinate $x_i$ of $x$.

Tight frames and ETFs have come up in different contexts generally in the study of neural networks and specifically in the analysis of superposition (see, e.g.,~\cite{Anthropic22,scherlis2022polysem,ivanov2026spectral,borobia2026linear,cowsik2024persian,liu2025superposition,papyan2020prev}).  We analyze these connections in more detail in Section~\ref{sec:model}.

Sparsity provides a well-established model of data that facilitates the study of superposition. 
Recalling our toy example from earlier, we cannot represent all three dimensions of color in the plane. However, there is no difficulty representing a large palette of $2$-sparse colors; sparsity is key to this dimensionality reduction.
Fortuitously, many datasets of interest can be decomposed as sparse linear combinations of feature vectors or dictionary atoms. 
This observation underlies not only classical compression algorithms like JPEG~\cite{pennebaker1992jpeg} and most of the compressed sensing literature (see, e.g.,~\cite{foucart2013math,elad2010sparse}), but also work on superposition in neural networks~\cite{sharkey2022taking,Anthropic22}.

Anthropic numerically probed a number of regimes of superposition arising in the training of neural networks and created a toy model of superposition based on their results~\cite{Anthropic22}.
To test for meaningful concepts in superposition in a neural network, we adopt Anthropic's toy model, which can be succinctly represented as
\[
\hat{x} :=\ReLU(W^\top Wx + b).
\]
By assumption, the vector of neuron activations is well approximated by a sparse non-negative linear combination $Wx$ of dictionary vectors.
Given this, we apply $W^\top$ to analyze the combination before applying a bias vector $b$ and the activation function ReLU in order to obtain an estimate $\hat{x}$ of the underlying sparse coefficients $x$. 
The focal point of superposition is how well $\hat{x}$ approximates $x$.

\subsection{Outline}
We begin in Section~\ref{sec:model} by presenting various types of superposition---differentiating between the mathematical models---and discussing other related work. 
In particular, our model of superposition is based on \textit{support recovery} and is simpler than the approximate sparse vector recovery used in~\cite{Anthropic22}, but leads to provable sufficient conditions on $W$, $x$, and $b$ for support recovery in superposition to occur.
Furthermore, these sufficient conditions align with properties that emerge in numerical experiments~\cite{liu2025superposition,cowsik2024persian,Anthropic22} with a more complicated model of superposition.

Our results have two flavors: In our \textit{random support model}, $x$ is a random vector with a certain expected sparsity level, while in our \textit{worst-case support model}, we let $x$ be any vector of a certain sparsity level. Coarsely, we can accommodate superposition in almost all signals of support order $d$ (random support model) and all signals of support order $\sqrt{d}$ (worst-case support model).  
We start with the random support model. In Section~\ref{sec:randvec}, we draw $x$ randomly with Bernoulli i.i.d.\ entries with parameter $p$.  We show in Theorem~\ref{thm:randvec} that if $p$ is upper bounded by quantities arising from the entries of $W^\top W$ and the spectral norm of $W$, then we have superposition with high probability.
We move on to proving some key lemmas in Section~\ref{sec:keylem} that identify when superposition occurs in the worst-case support model.  
The first main theorem in the worst-case support model concerns superposition when $W$ is random.  Namely, in Theorem~\ref{thm:maingaussian} in Section~\ref{sec:randW} we prove that when $W$ is a standard Gaussian i.i.d.\ matrix, superposition occurs with high probability when $\norm{x}_0 \lesssim\sqrt{d/\log n}$. 
In Section~\ref{sec:ETFs}, we probe the case that $W$ is an ETF. The main result is Theorem~\ref{thm:mainETF}, where it is shown that superposition occurs when $\norm{x}_0 \lesssim \sqrt{d}$. 
Along the way, we prove a structural result concerning ETFs (Proposition~\ref{prop:ETFevensp} and Remark~\ref{rem:numpm}) that might be of independent interest to those studying frame theory; namely, we show that if an ETF has $n > d+1$, then there is at most one column vector whose inner products with the other vectors are predominantly positive or negative.  
In other words, the sign patterns in the rows of the Gram matrix are relatively balanced.  

\section{Models of Superposition}\label{sec:model}

In a toy model of superposition, the key question is:
\begin{quote}
When does a $d \times n$ weight matrix $W$ with $d < n$ and bias vector $b \in \bR^n$ allow
\[
\ReLU(W^\top W x + b)
\]
to be a ``good'' representation of any ``sparse'' vector $x\in \bR^n$?
\end{quote}
Our approach to this question is simple enough to allow for theory in different regimes, while simultaneously aligning with the heuristic results seen in other, more complex models, as described below.

\subsection{The mainstream approach: Mean squared error}

In Anthropic's work~\cite{Anthropic22} and in work by others (e.g.,~\cite{liu2025superposition,cowsik2024persian}), the entries of $x$ are chosen independently with $x_i = c_i u_i$, $c_i \sim \textrm{Bernoulli}(p_i)$, $u_i \sim \textrm{Uniform}([0,a])$, for some $a >0$, and the goal is to learn $W$ and $b$ that minimize the \textit{mean squared error}:
\[
\bE_x \norm{x-\ReLU(W^\top Wx + b)}_2^2.
\]
For $n$ large enough compared to $d$ and $\sum_i p_i$ sufficiently small, $W^\top W$ and $b$ have similar emergent properties across different numerical experiments. In \cite{Anthropic22}, the Bernoulli probabilities are equal and in certain cases, the learned $W$ resembles an ETF, while in~\cite{liu2025superposition}, the restriction of the learned $W$ to columns corresponding to features with (relatively) larger $p_i$ is ETF-like. The authors of~\cite{cowsik2024persian} heuristically analyze a toy superposition model that decouples the encoding matrix $W_{\text{in}}$ from the decoding matrix $W_{\text{out}}^\top$. Despite this decoupling, they find that asymptotically, the product $W_{\text{out}}^\top W_{\text{in}}$ has equal diagonal entries and off-diagonal entries that are small in absolute value and have zero mean. 
(N.B., this resembles the Gram matrix $W^\top W$ in cases where $W$ is an approximately centered ETF or a matrix with i.i.d.\ standard Gaussian entries.) In these papers, the learned $b$ is approximately $\beta \mathbbm{1}$ for small $\beta < 0$ and $\mathbbm{1}$ the all--ones vector.

\subsection{Our approach: Support recovery}

Instead of minimizing the mean squared error between $\ReLU(W^\top W x + b)$ and $x$, we ask that these vectors have the \textit{same support}.
In order for this to even be feasible, we prevent $x$ from having any small nonzero entries by forcing $x \in \{0,1\}^n$.
In this setting, we consider two different models for $x$:
\begin{itemize}
\item[]\hspace{-0.2in}
\textbf{Random support model.}
$x$ is a vector with i.i.d.\ $\operatorname{Bernoulli}(p)$ entries for some $p\in(0,1)$.
\item[]\hspace{-0.2in}
\textbf{Worst-case support model.}
$x$ is \textit{any} vector in $\{0,1\}^n$ with at most $k$ nonzero entries.
\end{itemize}
We show that under both models, an appropriate choice of $b$ delivers \textit{support recovery}, i.e.,
\[
\operatorname{supp}\big(\ReLU(W^\top W x + b)\big)
=\operatorname{supp}(x),
\]
provided $W$ resembles an ETF in some sense.

While support recovery may seem like a departure from the mainstream mean-squared-error approach, we highlight three attractive features of our alternative.

\medskip
\noindent
\textbf{Support recovery captures part of the essence of superposition.}
In settings where $W$ and $b$ \textit{empirically} perform well in the context of mean squared error, we are able to \textit{prove} that they perform support recovery.
Indeed, we prove support recovery under the random support model when
\begin{itemize}
\item $W$ is nearly tight with low coherence (Theorem~\ref{thm:randvec}), and 
\item in this setting, we also are able to bound $\norm{x-\ReLU(W^\top Wx + b)}_{\infty}$ (Theorem~\ref{thm:randvec}).
\end{itemize}
We also prove support recovery under the worst-case support model when
\begin{itemize}
    \item $W$ has Gaussian i.i.d.\ entries (Theorem~\ref{thm:maingaussian}), or
    \item $W$ is an ETF  (Theorem~\ref{thm:mainETF}).
\end{itemize} 
In all of these cases, analysis of the proofs yields that $W^\top W$ has diagonal elements approximately equal and off-diagonal elements small and approximately equal in absolute value, while $b$ is approximately a small negative multiple of $\mathbbm{1}$, aligning with the numerical results of the more complex superposition models. In other words, $W^\top W$ is a low-rank approximation of the $n \times n$ identity and $b$ offsets the noise introduced since $W^\top W$ is not the identity.  

The geometric intuition behind this is shown in Figure~\ref{fig:R2super}.  We start with a $1$-sparse $x \in \{0,1\}^2$, i.e., $(1,0)$ or $(0,1)$. Then, multiplying by $W^\top W$ slightly perturbs the $1$-sparse $x$ so that it is no longer $1$-sparse (in blue, left).  By adding a small negative multiple $b$ of $\mathbbm{1}$ (in green, middle left), the entries of $W^\top W x +b$ with indices not in the support of $x$ are now negative and will be mapped to $0$ under $\ReLU$ (in orange, middle right). For $i$ in the support of $x$, $\ReLU(W^\top W x + b)$ is not too far from $1$ and can be approximated by a small scaling (in red, right). The modification from the middle right plot to the right plot in Figure~\ref{fig:R2super} also gives some intuition why similar $W$'s and $b$'s yield superposition in our model and the model from Anthropic. In particular, since $\ReLU$ is positively homogeneous, one may simply take $W\leftarrow \sqrt{\alpha}\cdot W$ and $b\leftarrow \alpha\cdot b$ for the appropriate choice of $\alpha > 0$ for $\ReLU(W^\top Wx + b)$ to approximately recover $x$.  A visualization of superposition in higher dimensions appears in Figure~\ref{fig:lollipops}.  Here, $x$ is a $6$-sparse vector in $\{0,1\}^{50}$ (in black, top left).  $W$ is chosen to be a centered $(25,50)$-ETF, and $W^\top Wx$ slightly perturbs $x$ (in blue, top right).  Also note that the entries of $W^\top W x$ which are above $6/7$ (denoted by the dashed blue line, top right) are precisely the entries in the support of $x$.  Thus, for $b = -6/7 \cdot \mathbbm{1}$, $\ReLU(W^\top W x + b)$ has the same support as $x$ (in orange, bottom left).  Shifting by the negative $b$ decreased each entry; so, one may scale (e.g., by $\alpha = 7/3$ in this case) the entries to approximate $x$ by $\alpha\cdot \ReLU(W^\top W x + b)$ (in red, bottom right).

\medskip
\noindent
\textbf{Support recovery represents a natural first step in superposition theory.}
We view the support recovery as a first step, much like how support recovery predated sparse recovery, which in turn predated general compressed sensing.
Indeed, sparse support recovery has a long history in statistics~\cite{dorfman1943detection} and regression analysis (cf.\ \cite{miller2002subset} and sources therein), while sparse recovery came later, followed by the theoretical analysis of recovery of noisy and compressible vectors. (See \cite{foucart2013math,elad2010sparse} and the references therein.)

\medskip
\noindent
\textbf{A similar model already appears in the literature.}
In~\cite{borobia2026linear}, the authors are concerned with computation in superposition.  
Their model may be expressed as
\[
\ReLU(W^\top (Wx +c) -\tfrac{1}{2} \mathbbm{1})
\]
having the same support as $x$ for sparse $x \in \{0,1\}^n$.  
One may identify our $b$ with their $W^\top c -\tfrac{1}{2} \mathbbm{1}$, though not every $b$ takes this form when $n>d+1$.  
Also, their models for the sparsity of $x$ and the probability distributions for $W$ are distinct from ours.

\subsection{Related work}
In recent work using a different model of superposition~\cite{scherlis2022polysem} (namely, approximating $x^\top D x$ for full rank diagonal $D$ and sparse $x$ with $x^\top W^\top W x +b$ for $W$ low rank and $b$ scalar bias), it is shown that tight frames (called ``semiorthogonal'' and ``everything bagels'' in the paper) optimize the paper's measure of capacity of superposition. 
Additionally, the authors of~\cite{Anthropic22} define a measure called \emph{feature dimensionality} to quantify ``how much'' of a dimension is used to represent a given feature in superposition; it is shown in~\cite{ivanov2026spectral} that this measure is optimized when $W$ is a so-called tegum product of tight frames. Finally, the authors of~\cite{borobia2026linear} present a lower bound for linear readouts in superposition  with potentially different encoding and decoding matrices. They call the bound the ``biorthogonal Welch floor'' (known as the ``cross frame potential'' in frame theory~\cite{benedetto2003finite,aceska2022cross}) and note that when the encoding and decoding matrices are the same tight frame, the bound is saturated. (See~\cite{aceska2022cross} for a full characterization of when the bound is saturated.)

We end this section by noting that ETFs have already emerged in another aspect of neural network research.
\textit{Neural collapse} is an empirical phenomenon that arises in the terminal phase of training with a deep neural network: the final weight matrix resembles a simplex ETF~\cite{papyan2020prev}.
(Interestingly, the non-final constituent weight matrices of trained LLMs like GPT-2 have also been observed to exhibit ETF-like features~\cite{liu2025superposition}.)
In the \textit{unconstrained features model} of neural collapse introduced in~\cite{mixon2022neural}, the goal is to find a weight matrix $W$, a bias vector $b$, as well as feature representations $\{h_{ij}\}_{i\in[c],j\in[n]}$ of the training set (which consists of $n$ examples from each of $c$ classes) that together minimize the square loss
\[
\sum_{i=1}^c\sum_{j=1}^n \|(W^\top h_{ij}+b)-e_i\|_2^2.
\]
Here, $e_i$ is the $i$th standard basis element, i.e., the one-hot encoding of the $i$th class.
One may verify that the above loss function equals zero if
\begin{itemize}
\item[(1)] 
the columns of $W$ form the vertices of a regular simplex centered at the origin, 
\item[(2)] 
each $h_{ij}$ is an appropriate positive multiple of the $i$th column of $W$, and 
\item[(3)] 
$b$ is an appropriate positive multiple of the all--ones vector.
\end{itemize}
Despite certain qualitative similarities, we view the neural collapse phenomenon as distinct from superposition.

\begin{figure}
    \centering
    \includegraphics[page=1,width=0.23\linewidth]{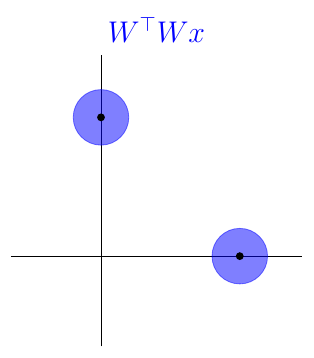}
    ~
    \includegraphics[page=2,width=0.23\linewidth]{SuperposVecPlot_mixon_edit_updated_thicker_titles.pdf}
    ~
    \includegraphics[page=3,width=0.23\linewidth]{SuperposVecPlot_mixon_edit_updated_thicker_titles.pdf}
    ~
    \includegraphics[page=4,width=0.23\linewidth]{SuperposVecPlot_mixon_edit_updated_thicker_titles.pdf}
    \caption{Intuition for superposition. Start with a nonnegative $k$-sparse vector $x$. In this case, $k=1$ and $x\in\{(1,0),(0,1)\}$. \textbf{Left:} Assume $W^\top W$ is close to the identity matrix in some sense, e.g., the columns of $W$ form an equiangular tight frame. Then $W^\top Wx$ can be interpreted as a noisy version of $x$, i.e., it resides in the blue disk about $x$. \textbf{Middle left:} Shift this disk by adding the bias vector~$b$, which tends to be close to a negative multiple of the all--ones vector. \textbf{Middle right:} At this point, we apply ReLU, which has the effect of projecting the disk onto the nonnegative orthant. In our case, the nonzero entries in $x$ were large relative to the noise added by applying $W^\top W$, and the effect is that every point in this projection has the same support as $x$. \textbf{Right:} If we care about recovering~$x$, not just its support, we can rescale the result so that it's close to $x$. Since ReLU is positively homogeneous, one can avoid rescaling after the ReLU by instead taking $W\leftarrow \sqrt{\alpha}\cdot W$ and $b\leftarrow \alpha\cdot b$.}
    \label{fig:R2super}
\end{figure}

\begin{figure}
    \centering
    \includegraphics[page=1,width=0.49\linewidth]{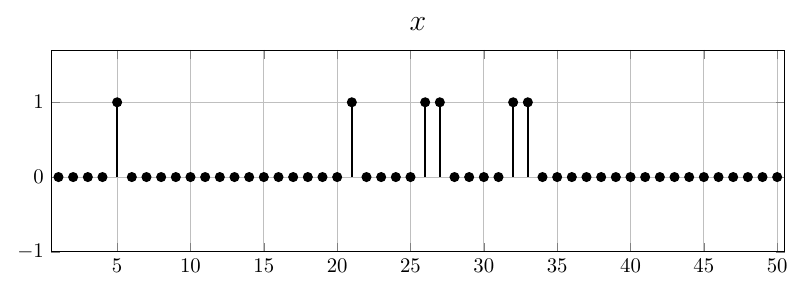}
    \includegraphics[page=2,width=0.49\linewidth]{lollipop_large_titles.pdf}
    \includegraphics[page=3,width=0.49\linewidth]{lollipop_large_titles.pdf}
    \includegraphics[page=4,width=0.49\linewidth]{lollipop_large_titles.pdf}
    \caption{Visualization of superposition. \textbf{Top left:} Start with a $k$-sparse binary vector $x$. In this case, $k=6$ and $x\in\mathbb{R}^{50}$. \textbf{Top right:} Fix a $25\times50$ equiangular tight frame $W$ that is \textit{centered}, meaning the column vectors sum to the zero vector. We plot the entries of $W^\top Wx$ and the threshold $\theta:=6/7$ as a dashed line. The entries above this threshold form the support of $x$. \textbf{Bottom left:} Taking $b:=-\theta\cdot\mathbbm{1}$, then $\operatorname{ReLU}(W^\top Wx+b)$ has the same support as $x$, but the entries are smaller.
    \textbf{Bottom right:} Multiplying by $\alpha:=7/3$ makes the nonzero entries closer to $1$. Since ReLU is positively homogeneous, one can avoid rescaling after the ReLU by instead taking $W\leftarrow \sqrt{\alpha}\cdot W$ and $b\leftarrow \alpha\cdot b$.}
    \label{fig:lollipops}
\end{figure}

\section{Random Support Model}\label{sec:randvec}

We begin by analyzing the case that $x$ is a Bernoulli i.i.d.\ vector with parameter $p$.  
In this case, we recover the support of $x$ with high probability assuming that $p$ is bounded by certain formulas involving the coherence of $W$ as well as its spectral norm.  
As in the toy models of superposition found in~\cite{Anthropic22,liu2025superposition,cowsik2024persian}, we also consider how well $\ReLU(W^\top W x +b)$ directly approximates $x$; however, their loss function is based on expected squared $\ell^2$ error, while we measure mismatch in the $\ell^{\infty}$ norm and prove a high-probability guarantee.

\begin{theorem}\label{thm:randvec}
Fix $p\in(0,1)$, and define $b\in\mathbb{R}^n$ by
\[
b_i
:=-\frac{1}{2}\|w_i\|^2-p\sum_{\substack{j=1\\j\neq i}}^n\langle w_i,w_j\rangle.
\]
Suppose the entries of $x$ are independent Bernoulli random variables with parameter $p$, and take
\[
\hat{x}:=2\cdot\operatorname{ReLU}(W^\top Wx+b).
\]
Consider the dictionary parameters 
\[
\alpha
:=\min_{i\in[n]}\|w_i\|^2,
\qquad
\mu
:=\max_{\substack{i,j\in[n]\\i\neq j}}|\langle w_i,w_j\rangle|,
\]
and let \(\varepsilon>0\) be given.
\begin{itemize}
\item[(a)]
With probability $\geq1-\varepsilon$, it holds that $\hat{x}$ has the same support as $x$, provided
\[
p
\leq\frac{1}{\|W\|_{2\to2}^2}\cdot\bigg(\frac{\alpha}{8\log(2n/\varepsilon)}-\frac{\mu}{6}\bigg).
\]
\item[(b)]
Suppose in addition that $\|w_i\|=1$ for all $i\in[n]$, and fix $\delta\in(0,1)$.
Then with probability $\geq1-\varepsilon$, it simultaneously holds that $\hat{x}$ has the same support as $x$ and $\|\hat{x}-x\|_\infty\leq\delta$, provided
\[
p
\leq\frac{1}{\|W\|_{2\to2}^2}\cdot\bigg(\frac{\delta^2}{8\log(2n/\varepsilon)}-\frac{\delta\mu}{6}\bigg).
\]
\end{itemize}
\end{theorem}

\begin{rem}
\label{rem.interpretation}
To interpret Theorem~\ref{thm:randvec}, consider the case where $W$ is a \textit{unit norm tight frame} of somewhat low coherence, by which we mean
\[
\|W\|_{2\to2}^2=\frac{n}{d},
\qquad
\alpha=1,
\qquad
\mu\ll\frac{1}{\log n}.
\]
Then Theorem~\ref{thm:randvec}(a) gives that $\operatorname{ReLU}(W^\top Wx+b)$ has the same support as $x$ with high probability provided $p=k/n$ and $k\ll d/\log n$.
(Here, $k$ represents the expected sparsity level of $x$.)
This is in stark contrast with the worst-case analysis in the subsequent sections that requires $x$ to have far fewer nonzero entries: $k=O(\sqrt{d})$.
This difference between average-case and worst-case performance (and for that matter, our proof of Theorem~\ref{thm:randvec}) rhymes with various observations previously made in~\cite{bajwa2012two} in the context of compressed sensing.
If in addition, $W$ is centered in the sense that its column vectors sum to the zero vector, then the bias vector in Theorem~\ref{thm:randvec} reduces to a negative multiple of the all-ones vector: $b=(-\frac{1}{2}+\frac{k}{n})\mathbbm{1}$.
This matches the superposition behavior observed in previous literature~\cite{liu2025superposition,cowsik2024persian}.
\end{rem}

\begin{proof}[Proof of Theorem~\ref{thm:randvec}]
For each $i\in[n]$, it holds that
\[
(W^\top Wx+b)_i
=\sum_{j=1}^n\langle w_i,w_j\rangle x_j+b_i
=\|w_i\|^2\cdot\bigg(x_i-\frac{1}{2}\bigg)+\sum_{\substack{j=1\\j\neq i}}^n\langle w_i,w_j\rangle\cdot(x_j-p).
\]
Thus, $\hat{x}$ has the same support as $x$ when
\[
\bigg|\sum_{\substack{j=1\\j\neq i}}^n\langle w_i,w_j\rangle\cdot(x_j-p)\bigg|
<\frac{1}{2}\|w_i\|^2
\qquad
\forall\,i\in[n],
\]
and in the case where $\|w_i\|=1$ for all $i\in[n]$, we may further conclude $\|\hat{x}-x\|_\infty\leq\delta$ when
\[
\bigg|\sum_{\substack{j=1\\j\neq i}}^n\langle w_i,w_j\rangle\cdot(x_j-p)\bigg|
<\frac{\delta}{2}
\qquad
\forall\,i\in[n].
\]
Fix $i\in[n]$, and denote the random variable $Z_j:=\langle w_i,w_j\rangle\cdot(x_j-p)$. 
Then for every $j\neq i$, we have $|Z_j|\leq\mu$ almost surely with $\mathbb{E}[Z_j]=0$ and $\mathbb{E}[Z_j^2]=p(1-p)\cdot\langle w_i,w_j\rangle^2$.
Observe that
\[
\sum_{\substack{j=1\\j\neq i}}^n\mathbb{E}[Z_j^2]
\leq p\sum_{\substack{j=1\\j\neq i}}^n\langle w_i,w_j\rangle^2
\leq p\sum_{j=1}^n\langle w_i,w_j\rangle^2
=p\|W^\top w_i\|^2
\leq p\|W\|_{2\to2}^2\|w_i\|^2.
\]
Then Bernstein's inequality gives
\[
\mathbb{P}\bigg\{
\bigg|\sum_{\substack{j=1\\j\neq i}}^n\langle w_i,w_j\rangle\cdot(x_j-p)\bigg|\geq t\bigg\}
=\mathbb{P}\bigg\{
\bigg|\sum_{\substack{j=1\\j\neq i}}^nZ_j\bigg|\geq t\bigg\}
\leq 2\operatorname{exp}\bigg(-\frac{\frac{1}{2}t^2}{p\|W\|_{2\to2}^2\|w_i\|^2+\frac{1}{3}\mu t}\bigg).
\]
For (a), take $t:=\frac{1}{2}\|w_i\|^2$ and apply the union bound over $i\in[n]$.
For (b), take $t:=\frac{\delta}{2}$ and apply the union bound over $i\in[n]$.
\end{proof}

\section{Worst-Case Support Model}\label{sec:keylem}
In the remainder of the paper, we consider the worst-case support model.  
Namely, we provide sufficient conditions for $\ReLU(W^\top Wx +b)$ to have the same support as $x$ whenever $\norm{x}_0 \leq k$.
In this section, we provide some key lemmas that will allow us to analyze two different settings: when the entries of $W$ are i.i.d.\ Gaussian (Section~\ref{sec:randW}) and when $W$ is an ETF (Section~\ref{sec:ETFs}).
We begin with a definition.

\begin{definition}
Given $W\in\mathbb{R}^{d\times n}$, let $k(W)$ denote the largest $k$ for which there exists $b\in\mathbb{R}^n$ such that for every $x\in\{0,1\}^n$ with $\|x\|_0\leq k$, it holds that $\operatorname{ReLU}( W^\top Wx + b )$ has the same support as $x$. 
\end{definition}

In classical compressed sensing, winning is equivalent to $W$ satisfying the null space property \cite{cohen2009compressed}, but this property is difficult to verify for a given $W$. 
By contrast, in our setting, winning is equivalent to an easy-to-verify property.
Indeed, in what follows, we characterize $k(W)$ in terms of certain computable statistics.
Given a real number $x$, let $x_+=\operatorname{ReLU}(x)$ denote its positive part, and let $x_-=\operatorname{ReLU}(-x)$ denote its negative part.
In particular, $x=x_+-x_-$. 
Finally, given a real vector $v$, let $\Sigma^{(k)}v$ denote the sum of the $k$ greatest (i.e., most positive) entries of $v$.

\begin{lemma}
\label{lem.char}
The inequality $k\leq k(W)$ is equivalent to
\[
\Sigma^{(k)}(\langle w_i,w_j\rangle_+)_{j\in[n]\setminus\{i\}}
+\Sigma^{(k-1)}(\langle w_i,w_j\rangle_-)_{j\in[n]\setminus\{i\}}
<\|w_i\|^2
\qquad
\forall\, i\in[n].
\]
\end{lemma}

\begin{proof}
For ($\Rightarrow$), suppose $k\leq k(W)$.
Then there exists $b\in\mathbb{R}^n$ such that for every $x\in\{0,1\}^n$ with $\|x\|_0\leq k$, it holds that $\operatorname{ReLU}( W^\top Wx + b )$ has same support as $x$.
Fix $i\in[n]$.
Then for every $S\subseteq[n]\setminus\{i\}$ with $|S|\leq k$, taking $x:=\sum_{j\in S}e_j$ gives
\[
\operatorname{ReLU}( W^\top Wx + b )_i
=0,
\]
i.e.,
\[
\sum_{j\in S}\langle w_i,w_j\rangle+b_i
=(W^\top Wx + b )_i
\leq 0.
\]
Meanwhile, for every $T\subseteq[n]\setminus\{i\}$ with $|T|\leq k-1$, taking $x:=\sum_{j\in T\cup\{i\}}e_j$ gives
\[
\operatorname{ReLU}( W^\top Wx + b )_i
>0,
\]
i.e.,
\[
\|w_i\|^2+\sum_{j\in T}\langle w_i,w_j\rangle+b_i
=( W^\top Wx + b )_i
>0.
\]
By isolating $-b_i$ in both inequalities, it follows that
\[
\sum_{j\in S}\langle w_i,w_j\rangle
\leq -b_i
<\|w_i\|^2+\sum_{j\in T}\langle w_i,w_j\rangle.
\]
Meanwhile,
\[
\max_{\substack{S\subseteq[n]\setminus\{i\}\\|S|\leq k}}
\sum_{j\in S}\langle w_i,w_j\rangle
=\Sigma^{(k)}(\langle w_i,w_j\rangle_+)_{j\in[n]\setminus\{i\}},
\quad
\min_{\substack{T\subseteq[n]\setminus\{i\}\\|T|\leq k-1}}\sum_{j\in T}\langle w_i,w_j\rangle
=-\Sigma^{(k-1)}(\langle w_i,w_j\rangle_-)_{j\in[n]\setminus\{i\}},
\]
and so the claimed inequality follows.

For ($\Leftarrow$), take $b\in\mathbb{R}^n$ defined by
\[
b_i
:=-\Sigma^{(k)}(\langle w_i,w_j\rangle_+)_{j\in[n]\setminus\{i\}}.
\]
Then for every $S\subseteq[n]\setminus\{i\}$ with $|S|\leq k$, taking $x:=\sum_{j\in S}e_j$ gives
\[
(W^\top Wx + b )_i
=\sum_{j\in S}\langle w_i,w_j\rangle-\Sigma^{(k)}(\langle w_i,w_j\rangle_+)_{j\in[n]\setminus\{i\}}
\leq 0,
\]
and so $\operatorname{ReLU}( W^\top Wx + b )_i=0$.
Meanwhile, for every $T\subseteq[n]\setminus\{i\}$ with $|T|\leq k-1$, taking $x:=\sum_{j\in T\cup\{i\}}e_j$ gives
\begin{align*}
( W^\top Wx + b )_i
&=\|w_i\|^2+\sum_{j\in T}\langle w_i,w_j\rangle-\Sigma^{(k)}(\langle w_i,w_j\rangle_+)_{j\in[n]\setminus\{i\}}\\
&>\|w_i\|^2+\sum_{j\in T}\langle w_i,w_j\rangle+\Big(\Sigma^{(k-1)}(\langle w_i,w_j\rangle_-)_{j\in[n]\setminus\{i\}}
-\|w_i\|^2\Big)\\
&\geq0,
\end{align*}
and so $\operatorname{ReLU}( W^\top Wx + b )_i>0$.
Since every $x\in\{0,1\}^n$ with $\|x\|_0\leq k$ has the property that $\operatorname{ReLU}( W^\top Wx + b )$ has same support as $x$, we conclude that $k\leq k(W)$.
\end{proof}

Note that 
\[
b_i
=-\Sigma^{(k)}(\langle w_i,w_j\rangle_+)_{j\in[n]\setminus\{i\}}
\]
results in a bias vector $b$ that has all non-positive entries.  
Furthermore, if each row of $W^\top W$ has similar structure, then each entry is approximately equal.

Next, we bound $k(W)$ using a classically studied concept: diagonal dominance.
Recall that an $n \times n$ matrix $A$ is \emph{strictly diagonally dominant} if for every $i \in [n]$,
\[
\abs{A_{i,i}} > \sum_{j \in [n] \backslash \{i\}} \abs{A_{i,j}}.
\]
This suggests a definition.

\begin{definition}
Given $W \in \bR^{d\times n}$, let $k'(W)$ denote the largest $k$ such that every principal $k \times k$ submatrix of $W^\top W$ is strictly diagonally dominant.
\end{definition}

Equivalently, $k'(W)$ is the largest integer $k$ such that
\[
\Sigma^{(k-1)}(\absip{w_i}{w_j})_{j \in [n]\setminus\{i\}}
<\|w_i\|^2
\qquad
\forall\, i\in[n].
\]
In the case where $\|w_i\|=1$ for every $i\in[n]$, one can relate $k'(W)$ to Tropp's Babel function~\cite{tropp2004greed}:
\[
\mu_1^{(k)}(W)
:=\max_{i\in[n]}\max_{\substack{S\subseteq[n]\setminus\{i\}\\|S|\leq k}}
\sum_{j\in S}|\langle w_i,w_j\rangle|.
\]
Specifically, $k'(W)$ is the largest integer $k$ for which $\mu_1^{(k-1)}(W)<1$.

Notably, $k'(W)$ is a much more intuitive parameter than $k(W)$, and conveniently, these are always within a factor of $2$ of each other:

\begin{prop}\label{prop:kk'}
Given $W \in \bR^{d\times n}$,
\[
\bigg\lfloor \frac{k'(W)}{2} \bigg\rfloor \leq k(W) \leq k'(W).
\]
\end{prop}
\begin{proof}
We first prove that $k(W) \leq k'(W)$.  Regardless of the sign pattern of the entries of $W^\top W$, it follows from Lemma~\ref{lem.char} that for all $i \in [n]$,
\[
\Sigma^{(k(W)-1)}(\absip{w_i}{w_j})_{j \in [n]\setminus\{i\}} \leq 
\Sigma^{(k(W))}(\ip{w_i}{w_j}_+)_{j\in[n]\setminus\{i\}}
+\Sigma^{(k(W)-1)}(\ip{w_i}{w_j}_-)_{j\in[n]\setminus\{i\}}
<\norm{w_i}^2.
\]
Thus, $W^\top W$ is strictly diagonally dominant for all $k(W) \times k(W)$ principal submatrices.

We now show that $\big\lfloor \frac{k'(W)}{2} \big\rfloor \leq k(W)$.  
For every $i \in [n]$, the definition of $k'(W)$ gives
\begin{align*}
\norm{w_i}^2 &> \Sigma^{(k'(W)-1)}(\absip{w_i}{w_j})_{j \in [n]\setminus\{i\}} \\[1pt]
&\geq \Sigma^{(2\lfloor k'(W)/2 \rfloor-1)}(\absip{w_i}{w_j})_{j \in [n]\setminus\{i\}} \\[1pt]
&\geq \Sigma^{(\lfloor k'(W)/2 \rfloor)}(\ip{w_i}{w_j}_{+})_{j \in [n]\setminus\{i\}} + \Sigma^{(\lfloor k'(W)/2 \rfloor-1)}(\ip{w_i}{w_j}_{-})_{j \in [n]\setminus\{i\}}.
\end{align*}
So, we now apply Lemma~\ref{lem.char} to conclude that $\big\lfloor \frac{k'(W)}{2} \big\rfloor \leq k(W)$.
\end{proof}

Both the upper and lower bounds in Proposition~\ref{prop:kk'} are tight.  The tightness of the upper bound follows in Lemma~\ref{lem:upptight}. The tightness of the lower bound appears in the following section in Theorem~\ref{thm:lowersat}.  In both cases, ETFs can serve as the $W$, where the upper bound is saturated by ETFs with $n\leq d+1$ and the lower bound with ETFs with $n > d+1$.

\begin{lem}\label{lem:upptight}
Let $W$ be such that $\ip{w_i}{w_j} \leq 0$ for all $i \neq j$.  Then $k(W)=k'(W)$.
\end{lem}
\begin{proof}
    Since $\ip{w_i}{w_j} \leq 0$, for any $k \leq n-1$ and $i \in [n]$, 
    \begin{align*}
\Sigma^{(k-1)}(\absip{w_i}{w_j})_{j \in [n]\setminus\{i\}}&=\Sigma^{(k-1)}(\ip{w_i}{w_j}_-)_{j\in[n]\setminus\{i\}}\\
&=\Sigma^{(k)}(\ip{w_i}{w_j}_+)_{j\in[n]\setminus\{i\}}+\Sigma^{(k-1)}(\ip{w_i}{w_j}_-)_{j\in[n]\setminus\{i\}}.
\end{align*}
Thus, it follows from Lemma~\ref{lem.char} that $k(W)=k'(W)$.
\end{proof}

In the following two sections, we apply the above ideas to prove different regimes when superposition holds. 
What follows are our two main results in this vein.

\begin{thm}\label{thm:maingaussian}
Suppose $n\geq d\geq100$ and $W\in\mathbb{R}^{d\times n}$ has i.i.d.\ $N(0,1)$ entries.
Then
\[
\left\lfloor\frac{1}{2}\sqrt{\frac{d}{\log(40n^2)}}\right\rfloor
\,\leq\, k'(W)
\,\leq\, \left\lceil \frac{3}{2}\sqrt{5d}\right\rceil
\]
with probability at least $90\%$.
\end{thm}
\begin{thm}\label{thm:mainETF}
Let $W$ be a $(d,n)$-ETF with $n> d+1$.  Then $k(W)=\big\lfloor \frac{k'(W)}{2} \big\rfloor$ and $k'(W)=\big\lceil\sqrt{\frac{d(n-1)}{n-d}}\big\rceil$.
\end{thm}

\section{Proof of Theorem~\ref{thm:maingaussian}}\label{sec:randW}
We begin our analysis of the worst--case support model by assuming that $W$ is a Gaussian matrix.  In this case, we are able to prove lower (Proposition~\ref{prop:randkW1}) and upper (Proposition~\ref{prop:randkW2}) bounds on maximum sparsity levels that yield support recovery in superposition with high probability.

\begin{prop}\label{prop:randkW1}
Suppose $W\in\mathbb{R}^{d\times n}$ has i.i.d.\ $N(0,1)$ entries.
Then for any $\delta > 0$,
\[
k'(W)\geq \left\lfloor\frac{1}{2}\sqrt{\frac{d}{\log(2n^2/\delta)}}\right\rfloor
\]
with probability at least $1-\delta$.
\end{prop}

\begin{proof}
We will use the standard fact that, for each $i,j\in[n]$ with $i\neq j$, it holds that
\[
\ip{\tfrac{w_i}{\|w_i\|}}{w_j}
\sim N(0,1),
\qquad
\|w_i\|
\sim \chi(d).
\]
Indeed, the distribution of $w_j$ is rotation invariant, and its first coordinate (say) has distribution $N(0,1)$.
Meanwhile, $\chi(d)$ is \textit{defined} to be the Euclidean norm of a standard Gaussian vector in $\mathbb{R}^d$.
This motivates the following choice of union bound, where $a>0$ is a parameter to be selected later:
\begin{align*}
\bP\left\{ k'(W) < k \right\} &= \bP \left\{ \exists i \in [n] \enskip \textrm{s.t.} \enskip \Sigma^{(k-1)}\left(\absip{w_i}{w_j}\right)_{j \in [n]\setminus\{i\}} \geq \norm{w_i}^2 \right\}\\
&\leq \sum_{i \in [n]} \bP \left\{ \Sigma^{(k-1)}\left(\absip{\tfrac{w_i}{\norm{w_i}}}{w_j}\right)_{j \in [n]\setminus\{i\}} \geq \norm{w_i}\right\}\\
&\leq \sum_{i \in [n]}\left(
\bP\left\{\Sigma^{(k-1)}\left(\absip{\tfrac{w_i}{\norm{w_i}}}{w_j}\right)_{j \in [n]\setminus\{i\}} \geq a\right\}+\bP\left\{\|w_i\|\leq a\right\}\right)\\
&=n\cdot\bigg(\bP\Big\{
\Sigma^{(k-1)}(|Z_j|)_{j\in[n-1]}
\geq a\Big\}+\bP\left\{X\leq a\right\}\bigg),
\end{align*}
where $Z_1,\ldots,Z_{n-1}\sim N(0,1)$ are independent and $X\sim\chi(d)$.
For the first term, we apply the union bound and the Chernoff bound:
\begin{align*}
\bP\Big\{
\Sigma^{(k-1)}(|Z_j|)_{j\in[n-1]}
\geq a\Big\}
\leq\bP\Big\{(k-1)\max_{j\in[n-1]}|Z_j|\geq a\Big\}
&\leq \sum_{j=1}^{n-1}\bP\Big\{|Z_j|\geq\tfrac{a}{k-1}\Big\}\\
&\leq (n-1)\cdot 2 e^{-a^2/(2(k-1)^2)}.
\end{align*}
For the second term, we apply Lemma~1 from~\cite{laurent2000adaptive}, which states that $\bP\{X^2-d\leq -2\sqrt{dt}\}\leq e^{-t}$ for each $t>0$.
Taking $t:=\varepsilon^2 d$, then
\[
\bP\bigg\{X\leq \sqrt{(1-2\varepsilon)d}\bigg\}
=\bP\Big\{X^2-d\leq -2\varepsilon d\Big\}\leq e^{-\varepsilon^2 d}.
\]
This suggests taking $a:=\sqrt{(1-2\varepsilon)d}$, in which case
\[
\bP\left\{ k'(W) < k \right\} 
\leq n\cdot\bigg((n-1)\cdot 2\exp\Big(-\tfrac{(1-2\varepsilon)d}{2(k-1)^2}\Big)+\exp(-\varepsilon^2 d)\bigg).
\]
Taking $\varepsilon:=\frac{1}{2k}$, then since $k\geq 2$ without loss of generality, we have $1-2\varepsilon=1-\frac{1}{k}\geq\frac{1}{2}$, and so
\[
\bP\left\{ k'(W) < k \right\} 
\leq n\cdot\Big((n-1)\cdot 2e^{-d/(4(k-1)^2)}+e^{-d/(4k^2)}\Big)
\leq 2n^2e^{-d/(4k^2)},
\]
which implies the result.
\end{proof}

\begin{prop}\label{prop:randkW2}
Suppose $n\geq d\geq 100$ and $W\in\mathbb{R}^{d\times n}$ has i.i.d.\ $N(0,1)$ entries.
Then
\[
k'(W)
\leq\left\lceil\frac{3}{2}\sqrt{5d}\right\rceil
\]
with probability at least $1-2e^{-d/20}$.
\end{prop}

\begin{proof}
Let $m$ denote the median of the standard half-normal distribution.
A union bound gives
\begin{align*}
\bP\{k'(W)> k\}
&=\bP\Big\{\forall i\in[n],\,\Sigma^{(k)}(|\langle w_i,w_j\rangle|)_{j\in[n]\setminus\{i\}}<\|w_i\|^2\Big\}\\
&\leq\bP\Big\{\Sigma^{(k)}(|\langle \tfrac{w_n}{\|w_n\|},w_j\rangle|)_{j\in[n-1]}<\|w_n\|\Big\}\\
&\leq\bP\Big\{\Sigma^{(k)}(|Z_j|)_{j\in[n-1]}\leq km\Big\}+\bP\Big\{X\geq km\Big\},
\end{align*}
where $Z_1,\ldots,Z_{n-1}\sim N(0,1)$ are independent and $X\sim \chi(d)$.
For the first term, note that the inequality $\Sigma^{(k)}(|Z_j|)_{j\in[n-1]}\leq km$ implies that the $k$th largest of $(|Z_j|)_{j\in[n-1]}$ is at most $m$, which in turn is equivalent to having at least $n-k$ of these $n-1$ random variables being at most $m$.
Since $m$ is the median of the continuous random variable $|Z_j|$, then denoting $B\sim\operatorname{Binomial}(n-1,\frac{1}{2})$, we have
\[
\bP\Big\{\Sigma^{(k)}(|Z_j|)_{j\in[n-1]}\leq km\Big\}
\leq \bP\{B\geq n-k\}
\leq\exp\bigg(-\frac{(n-2k+1)^2}{2(n-1)}\bigg),
\]
where the last step applies Hoeffding's inequality under the assumption that $k-1\leq\frac{n-1}{2}$.
For the second term, we again appeal to Lemma~1 from~\cite{laurent2000adaptive}, which states that
\[
\bP\{X^2-d\geq 2\sqrt{dt}+2t\}
\leq e^{-t}
\]
for each $t>0$.
Take $t:=d$ and assume $k\geq\frac{3}{2}\sqrt{5d}$.
Then since $m=\sqrt{2} \operatorname{erf}^{-1}(\frac{1}{2}) \geq \frac{2}{3}$ (where $\operatorname{erf}^{-1}$ is the inverse error function), we have
\[
\bP\Big\{X\geq km\Big\}
\leq\bP\bigg\{X^2-d\geq 2\sqrt{dt}+2t\bigg\}
\leq e^{-t}
=e^{-d}.
\]
Together, we have
\[
\bP\Big\{k'(W)> \Big\lceil \tfrac{3}{2}\sqrt{5d}\Big\rceil\Big\}
\leq \exp\bigg(-\frac{(n-2\lceil \tfrac{3}{2}\sqrt{5d}\rceil+1)^2}{2(n-1)}\bigg)+e^{-d}
\leq 2e^{-d/20},
\]
where the last step uses the hypothesis that $n \geq d \geq 100$, which one may verify implies
\[
\frac{(n-2\lceil \tfrac{3}{2}\sqrt{5d}\rceil+1)^2}{2d(n-1)}
\geq\frac{1}{20}.
\qedhere
\]
\end{proof}

We are now ready to prove Theorem~\ref{thm:maingaussian}.
\begin{proof}[Proof of Theorem~\ref{thm:maingaussian}]
  This follows from a union bound over Propositions~\ref{prop:randkW1} and~\ref{prop:randkW2}.
\end{proof}

\section{Proof of Theorem~\ref{thm:mainETF}}\label{sec:ETFs}
In this section, we focus on proving that ETFs provide support recovery in superposition in the worst-case support model. We already proved that (centered) tight frames with low coherence yield not only support recovery but also low sup norm error in Remark~\ref{rem.interpretation}. The main result of this section, Theorem~\ref{thm:mainETF}, will show that ETFs yield support recovery for less sparse vectors than Gaussian $W$ (Theorem~\ref{thm:maingaussian}).  Additionally, we prove a structural result about real ETFs (Proposition~\ref{prop:ETFevensp} and Remark~\ref{rem:numpm}).

We begin by noting that we can easily calculate $k'(W)$ when $W$ is an ETF.
\begin{lem}\label{lem:k'ETF}
Given an ETF $W$ with coherence $\mu>0$, then $k'(W) = \lceil \frac{1}{\mu} \rceil$.
\end{lem}
\begin{proof}  
For each $i,k \in [n]$, we have $\norm{w_i}^2=1$ and
\[
\Sigma^{(k-1)}(\absip{w_i}{w_j})_{j \in [n]\setminus\{i\}} = (k-1) \mu.
\]
As such, $k'(W)$ is the largest integer $k$ for which $(k-1)\mu<1$, i.e., the largest integer strictly less than $\frac{1}{\mu}+1$, namely, $\lceil \frac{1}{\mu}\rceil$.
\end{proof}

If $W$ is an ETF with all non-positive inner products, then either $n=d$ (orthonormal basis) or $n=d+1$ (regular simplex ETF).  In what follows, we will show that all real ETFs with $n > d+1$ saturate the lower bound from Proposition~\ref{prop:kk'}, while proving certain stronger structural results about ETFs along the way. First, we need a proposition that shows that the number of positive and negative inner products in a real ETF with $n > d+1$ are fairly evenly distributed.  
\begin{prop}\label{prop:ETFevensp}
Consider a $(d,n)$-ETF $W$ with $n > d+1$ and coherence $\mu$ and set 
\[
C := \frac{1}{2} \left( \frac{n}{2} - 2 - \frac{\abs{n/d-2}}{2\mu}\right).
\]
Further set $s_i$ to be the number of negative off-diagonal entries of row $i$ of $W^\top W$ and $p_i=n-1-s_i$ is the number of positive off-diagonal entries of row $i$.
Then there is at most one row of $W^\top W$ with 
\[
\min\{ s_i, p_i \} \leq C.
\]
\end{prop}

\begin{proof}
Let us first recall an important consequence of the well-known correspondence between equiangular tight frames and strongly regular graphs~\cite{corneil1991geometry,waldron2009construction}.
Specifically, for every $i\in[n]$, there exists a \textit{switching} $D$, that is, an $n\times n$ diagonal matrix with diagonal entries in $\{\pm1\}$, such that the
$i$th row of $(WD)^\top(WD)$ has all positive off-diagonal entries, in which case every other row has exactly
\[
k_{\SRG}
:=\frac{n}{2}-1-\frac{n/d-2}{2\mu}
\]
negative off-diagonal entries.  
We will use this property of ETFs to show that two rows cannot both have a large imbalance of positive and negative off-diagonal entries.
Indeed, as we make rigorous below, if two rows had too much imbalance, then switching to normalize one of the rows would result in the wrong number of negative off-diagonal entries in the other.

Suppose to the contrary that there exist $i,j\in[n]$ with $i\neq j$ such that
\[
\min\{ s_i, p_i \} \leq C 
\qquad
\text{and}
\qquad
\min\{ s_j, p_j \} \leq C.
\]
Denote the majority sign of row $r\in\{i,j\}$ by $\tau_r:=\operatorname{sign}(p_r-s_r)$, and let $E_r$ denote the set of off-diagonal indices in row $r$ with the minority sign $-\tau_r$.
Then
\[
|E_r|
=\min\{p_r,s_r\} \qquad \Rightarrow \qquad |E_i|+|E_j|\leq 2C.
\]

Now switch so that the $i$th row has all positive off-diagonal entries.
For every $\ell\in[n]\setminus\{i,j\}$, the sign of the $(j,\ell)$-entry after switching is the so-called \textit{triple product}
\[
\operatorname{sign}\langle w_i,w_j\rangle\,
\operatorname{sign}\langle w_j,w_\ell\rangle\,
\operatorname{sign}\langle w_\ell,w_i\rangle,
\]
which in turn equals $\operatorname{sign}\langle w_i,w_j\rangle \cdot \tau_i\tau_j$ whenever $\ell\notin \{i,j\}\cup E_i\cup E_j$.
We use this to derive a contradiction in two cases.

\medskip

\noindent
\textbf{Case I:} $\operatorname{sign}\langle w_i,w_j\rangle \cdot \tau_i\tau_j=+1$.
Then every negative entry in the switched version of the $j$th row is indexed by a member of $E_i\cup E_j$, and so
\[
k_{\SRG}
\leq |E_i|+|E_j|
\leq 2C
= \frac{n}{2}-2-\bigg|\frac{n/d-2}{2\mu}\bigg|
< \frac{n}{2}-1-\frac{n/d-2}{2\mu}
= k_{\SRG},
\]
a contradiction.

\medskip

\noindent
\textbf{Case II:} $\operatorname{sign}\langle w_i,w_j\rangle \cdot \tau_i\tau_j=-1$.
Then every member of $[n]\setminus(\{i,j\}\cup E_i\cup E_j)$ indexes a negative entry in the switched version of the $j$th row, and so
\[
k_{\SRG}
\geq n-2-|E_i|-|E_j|
\geq n-2 - 2C
= \frac{n}{2}+\bigg|\frac{n/d-2}{2\mu}\bigg|
> \frac{n}{2}-1-\frac{n/d-2}{2\mu}
=k_{\SRG},
\]
a contradiction.
\end{proof}
\begin{rem}\label{rem:numpm}
Consider a $(d,n)$-ETF $W$ with coherence $\mu$ and set 
\[
B := \frac{n}{2}+ 1 +\abs{\frac{n/d -2}{2 \mu}}.
\]
Set $s_i$ to be the number of negative off-diagonal entries of row $i$ of $W^\top W$ and $p_i=n-1-s_i$ the number of positive off-diagonal entries of row $i$. For all but at most one $i \in [n]$,
\[
\abs{p_i - s_i} < B.
\]
This follows from Proposition~\ref{prop:ETFevensp} and
\[
\min\{ s_i, p_i \}
=\frac{(n-1)-|p_i-s_i|}{2}.
\]
We can interpret this as meaning that for any real ETF, for all but at most one $i \in [n]$, the cardinalities of $w_j$ which form acute versus obtuse angles with $w_i$ are approximately equal.
\end{rem}
There is one special case where we will need a stronger bound.
\begin{cor}\label{cor:numpm2}
Let $s_i$ and $p_i$ be as above. If $W$ is a $(3,6)$-ETF, then for at least one $i \in[n]$
\[
\min\{ s_i, p_i \} \geq 2.
\]
\end{cor}
\begin{proof}
    Consider a $(3,6)$-ETF.  If any row has $0$ negatives, then every remaining row has exactly $k_{\SRG}=2$ negatives and $3$ off-diagonal positives. Also, if any row has $0$ off-diagonal positives, then every remaining row has exactly $k_{\SRG}+1=3$ negatives and $2$ off-diagonal positives. So, assume by way of contradiction that each row has either exactly $1$ negative or exactly $1$ off-diagonal positive. Say $\ip{w_i}{w_j} < 0$ where row $i$ has exactly one negative and switch index $j$.  Then the new row $i$ has all positive entries, but new row $j$ has either $1$ or $4$ negative entries, a contradiction. A similar argument works for switching the sole positive off-diagonal entry of a row.  Thus, there is at least one row with at least two positives and two negatives.
\end{proof}

Although we ignored the case $n=d+1$ in the above analysis, i.e., when $W$ is a simplex ETF, in that case Proposition~\ref{prop:ETFevensp} holds vacuously since
\[
\frac{n}{2}+ 1 +\abs{\frac{n/d -2}{2 \mu}}= \frac{d+1}{2} +1 +\frac{2-(d+1)/d}{2(1/d)} = d+1
= n.
\]
Furthermore, for a simplex ETF, there exists a switching such that \textit{all} of the off-diagonal entries of $W^\top W$ are negative.

\begin{thm}\label{thm:lowersat}
Let $W$ be a $(d,n)$-ETF with $n> d+1$.   Then $\big\lfloor \frac{k'(W)}{2} \big\rfloor = k(W)$.
\end{thm}

\begin{proof}
The idea of the proof is to show that there is at least one row \(i\) of the Gram matrix \(W^{\top}W\) such that for any \(k\geq \lfloor k'(W)/2\rfloor+1\) the expression
\[
\Sigma^{(k)}(\ip{w_i}{w_j}_{+})_{j \in [n]\setminus\{i\}} + \Sigma^{(k-1)}(\ip{w_i}{w_j}_{-})_{j \in [n]\setminus\{i\}}
\]
has at least $k'(W)$ non-zero summands, i.e., sums to at least $\norm{w_i}^2$. Then Lemma~\ref{lem.char} yields $\big\lfloor \frac{k'(W)}{2} \big\rfloor \geq k(W)$ and equality will follow from Proposition~\ref{prop:kk'}. To that end, it suffices to show that for some row $i$, 
\[
s_i, p_i \geq \lfloor k'(W)/2 \rfloor +1.
\]
Thus, it follows from Proposition~\ref{prop:ETFevensp} that it suffices to show that for $C = \frac{1}{2} \left( \frac{n}{2} - 2 - \frac{\abs{n/d-2}}{2\mu}\right)$,
\[
 \lfloor C \rfloor +1 \geq \left\lfloor \frac{k'(W)}{2}\right\rfloor +1.
\]
We proceed in cases.

\medskip
\noindent
\textbf{Case I:} $n \neq 2d$, $n > d+1$.
Since $n \notin \{d+1, 2d\}$, $1/\mu$ is an odd integer~\cite{lemmmens1973equi} and by Lemma~\ref{lem:k'ETF},
\[
\lfloor k'(W)/2 \rfloor = \frac{1}{2\mu} -\frac{1}{2}.
\]
We begin by working backwards and using the Welch bound and the fact that $n >2$ and $n/d > 1$
\begin{align}
  \lefteqn{ \lfloor C\rfloor +1 \geq \left\lfloor \frac{k'(W)}{2}\right\rfloor +1}\label{eqn:nnot2d}\\
  &\Leftrightarrow \quad \left\lfloor\frac{n/2-2-\vert n/d-2\vert/2 \mu}{2}\right\rfloor \geq \frac{1}{2\mu} - \frac{1}{2}\nonumber\\
  &\Leftrightarrow \quad \frac{n/2-2-\vert n/d-2\vert/2 \mu}{2} \geq \frac{1}{2\mu} - \frac{1}{2}\label{eqn:nnot2da}\\
  &\Leftrightarrow \quad n/2-2-\frac{\vert n/d-2\vert}{2\mu}\geq \frac{1}{\mu} - 1 \nonumber\\
  &\Leftrightarrow \quad n-2\geq \frac{1}{\mu}\left(2+\vert n/d-2\vert\right). \label{eqn:nnot2db}
  \end{align}
The equivalence in~\eqref{eqn:nnot2da} holds since $1/(2\mu)-1/2$ is an integer. \\

Let $a$ and $b$ denote the reciprocals of the coherences of $W$ and its Naimark complement, respectively.
Putting $d':=n-d$, then
\[
a = \frac{1}{\mu}
=\sqrt{\frac{d(n-1)}{d'}},
\qquad
b =
\sqrt{\frac{d'(n-1)}{d}},
\]
and so
\[
ab=n-1,
\qquad
\frac{d'}{d}=\frac{b}{a},
\qquad
\frac{n}{d}=\frac{d+d'}{d}=1+\frac{b}{a}.
\]
Then in terms of $a$ and $b$,~\eqref{eqn:nnot2db} in the computation above
\[
n-2
\geq \frac{1}{\mu}\bigg(2+\Big|\frac{n}{d}-2\Big|\bigg)
\]
becomes
\[
ab-1
\geq2a+|b-a|.
\]
To establish this inequality, we will use the fact~\cite{lemmmens1973equi} that $a$ and $b$ are odd integers; in fact, since $n>d+1>2$, we have $a,b\geq 3$.
In the case where $d'>d$, we have $b>a$, and so the desired inequality becomes
\[
ab-1
\geq a+b,
\]
which follows from the fact that for $a,b\geq3$,
\[
(ab-1)-(a+b)
=(a-1)(b-1)-2
\geq2.
\]
In the other case where $d'<d$, we have $b<a$, and so the desired inequality becomes
\[
ab-1
\geq 3a-b,
\]
which similarly follows from the fact that for $a,b\geq3$,
\[
(ab-1)-(3a-b)
=a(b-3)+b-1
\geq2.
\]
Thus, the claim holds for all ETFs with $n \neq 2d$.

\medskip
\noindent
\textbf{Case II:}  $n = 2d$.
We proceed similarly to above, working backwards and computing using Lemma~\ref{lem:k'ETF}:
\begin{align}
  \lefteqn{ \lfloor C\rfloor +1 \geq \left\lfloor \frac{k'(W)}{2}\right\rfloor +1}\nonumber\\
  &\Leftrightarrow \quad \left\lfloor\frac{d-2-0}{2}\right\rfloor +1 \geq \left\lfloor \frac{\lceil \sqrt{2d-1}\rceil}{2}\right\rfloor +1\nonumber\\
  &\Leftrightarrow \quad  \left\lfloor\frac{d-2}{2}\right\rfloor \geq \left\lfloor \frac{\lceil \sqrt{2d-1}\rceil}{2}\right\rfloor \label{eqn:n2dcasez}\\
  &\Leftarrow \quad d-2 \geq \lceil\sqrt{2d-1}\rceil \label{eqn:n2dcasea}\\
  &\Leftrightarrow \quad d-2 \geq \sqrt{2d-1} \label{eqn:n2dcaseb}\\
  &\Leftrightarrow \quad d^2-6d+5 \geq 0 \nonumber\\
  &\Leftarrow \quad d \geq 5 \label{eqn:n2dcasec}.
\end{align}
Note that~\eqref{eqn:n2dcaseb} holds because $d-2$ is an integer and that~\eqref{eqn:n2dcasea} and~\eqref{eqn:n2dcasec} are left arrows.
It follows from Gerzon's bound~\cite{lemmmens1973equi} that there is no $(2,4)$-ETF, meaning we only have to test $(3,6)$ and $(4,8)$ individually. Plugging $d = 4$ into~\eqref{eqn:n2dcasez}, we obtain saturation, i.e.,
\[
\left\lfloor \frac{2}{2}\right\rfloor = \left\lfloor \frac{\left\lceil\sqrt{7}\right\rceil}{2}\right\rfloor=1.
\]
For  $(3,6)$, we use the stronger bound from Corollary~\ref{cor:numpm2}.  That is, for $(3,6)$, there is at least one row with
\[
\min\{s_i,p_i\} \geq 2= \left\lfloor \frac{1}{2} \left\lceil \sqrt{2(3)-1}\right\rceil \right\rfloor +1 = \left\lfloor \frac{k'(W)}{2}\right\rfloor +1.
\]
Thus, the claim holds for all ETFs with $n = 2d$. 
\end{proof}

We are ready to prove the final theorem.
\begin{proof}[Proof of Theorem~\ref{thm:mainETF}]
    This follows immediately from Lemma~\ref{lem:k'ETF}, the Welch--Rankin bound, and Theorem~\ref{thm:lowersat}. 
\end{proof}

\section*{Acknowledgments}
MI was supported in part by MATS Research and NSF DMS 2110745.
JJ was supported in part by NSF DMS 2220320.
EJK was supported in part by the U.S. Air Force Office of Scientific Research (AFOSR) and the Air Force Research Laboratory (AFRL) through the Summer Faculty Fellowship Program (SFFP).
DGM was supported in part by NSF DMS 2220304. 
The views expressed are those of the authors and do not reflect the official guidance or position of the United States Government, the Department of Defense, the United States Air Force, or the United States Space Force. On behalf of all authors, the corresponding author states that there is no conflict of interest. Data sharing is not applicable to this article as no datasets were generated or analyzed during the current study.

%\bibliography{superposition}{}
%\bibliographystyle{amsalpha}

\newcommand{\etalchar}[1]{$^{#1}$}
\providecommand{\bysame}{\leavevmode\hbox to3em{\hrulefill}\thinspace}
\providecommand{\MR}{\relax\ifhmode\unskip\space\fi MR }
\providecommand{\MRhref}[2]{%
  \href{http://www.ams.org/mathscinet-getitem?mr=#1}{#2}
}
\providecommand{\href}[2]{#2}

\end{document}